\documentclass[twoside,11pt]{article}

\usepackage{jmlr2e}

\firstpageno{1}

\usepackage{amsmath}
\usepackage{amssymb}
\usepackage{mathtools}
\usepackage[capitalize,noabbrev]{cleveref}
\usepackage{algorithm}%
\usepackage{algpseudocode}%

\usepackage[textsize=tiny]{todonotes}

\usepackage{comment}
\usepackage{natbib}
\usepackage{subfigure}
\usepackage{graphicx}
\usepackage{float}

\includecomment{commentediting}
\excludecomment{commentediting}

\includecomment{commentediting2}
\excludecomment{commentediting2}

\includecomment{commentediting3}
\excludecomment{commentediting3}

\begin{document}

\title{Quantum Computing Provides Exponential Regret Improvement in Episodic Reinforcement Learning}

\author{\name Bhargav Ganguly$^1$,  \name Yulian Wu$^2$, \name Di Wang$^2$, and \name Vaneet Aggarwal$^{1, 2}$\\
		\addr $^1$ Purdue University, West Lafayette, IN 47907, USA\\
		\addr $^2$ KAUST, Thuwal 23955, KSA\\
		\email bganguly@purdue.edu, yulian.wu@kaust.edu.sa, di.wang@kaust.edu.sa, vaneet@purdue.edu
}
	\if 0
	\addr Department of Statistics\\
	University of Washington\\
	Seattle, WA 98195-4322, USA
	\AND
	\name Michael I.\ Jordan \email jordan@cs.berkeley.edu \\
	\addr Division of Computer Science and Department of Statistics\\
	University of California\\
	Berkeley, CA 94720-1776, USA
\fi 
\editor{}

\maketitle

\begin{abstract}
 In this paper, we investigate the problem of \textit{episodic reinforcement learning} with  quantum oracles for state evolution. To this end, we propose an \textit{Upper Confidence Bound} (UCB) based quantum algorithmic framework to facilitate learning of a finite-horizon MDP. Our quantum algorithm achieves an exponential improvement in regret as compared to the classical counterparts, achieving a regret of $\Tilde{\mathcal{O}}(1)$ as compared to $\Tilde{\mathcal{O}}(\sqrt{K})$ \footnote{$\Tilde{\mathcal{O}}(\cdot)$ hides logarithmic terms.}, $K$ being the number of training episodes. In order to achieve this advantage, we exploit efficient quantum mean estimation technique that provides quadratic improvement in  the number of i.i.d. samples needed to estimate the mean of sub-Gaussian random variables as compared to classical mean estimation. This improvement is a key to the significant regret improvement in quantum reinforcement learning. We provide proof-of-concept experiments on various RL environments that in turn demonstrate performance gains of the proposed algorithmic framework.

\end{abstract}

\section{Introduction}\label{sec:intro}
\textit{Quantum Machine Learning} (QML) is an emerging domain built at the confluence of quantum information processing and \textit{Machine Learning} (ML) \citep{saggio2021experimental}. A noteworthy volume of prior works in QML demonstrate how quantum computers could be effectively leveraged to improve upon classical results pertaining to classification/regression based predictive modeling tasks \citep{aimeur2013quantum, rebentrost2014quantum, arunachalam2017guest}. While the efficiency of QML frameworks have been shown on conventional supervised/unsupervised ML use cases, how similar improvements could be translated to \textit{Reinforcement Learning} (RL) tasks have gained significant attention recently and is the focus of this paper.

Traditional RL tasks comprise of an agent interacting with an external environment attempting to learn its configurations, while collecting rewards via actions and state transitions \citep{sutton2018reinforcement}.  RL techniques have been credibly deployed at scale over a variety of agent driven decision making industry use cases, e.g., autonomous navigation in self-driving cars \citep{al2019deeppool}, recommendation systems in e-commerce websites \citep{rohde2018recogym}, and online gameplay agents such as AlphaGo \citep{silver2017mastering}. Given the wide applications, this paper aims to study if quantum computing can help further improve the performance of reinforcement learning algorithms. This paper considers an episodic setup, where the  learning occurs in episodes with a finite horizon. The performance measure for algorithm design is the \textit{ regret} of agent's rewards \citep{mnih2016asynchronous, cai2020provably}, which measures the gap in the obtained rewards by the Algorithm and the optimal algorithm. A central idea in RL algorithms is the notion of \textit{exploration-exploitation} trade-off, where agent's policy is partly constructed on its experiences so far with the environment as well as injecting a certain amount of optimism to facilitate exploring sparsely observed policy configurations \citep{kearns2002near, jin2020provably}. In this context, we emphasize that our work adopts the well-known \textit{Value Iteration} (VI) technique which combines empirically updating state-action policy model with \textit{Upper Confidence Bound} (UCB) based strategic exploration \citep{azar2017minimax}.

In this paper, we show an exponential improvement in the  regret of reinforcement learning. The key to this improvement is that quantum computing allows for  improved \textit{mean estimation} results over classical algorithms \citep{brassard2002quantum, hamoudi2021quantum}. Such mean estimators feature in very recent studies of \textit{quantum bandits} \citep{wang2021quantum}, \textit{quantum reinforcement learning} \citep{wang2021quantum2}, thereby leading to noteworthy convergence gains. In our proposed framework, we specifically incorporate a quantum information processing technique that improves non-asymptotic bounds of conventional empirical mean estimators which was first demonstrated in \citep{hamoudi2021quantum}.
In this regard, it is worth noting that a crucial novelty pertaining to this work is carefully engineering agent's interaction with the environment in terms of collecting classical and quantum signals. We further note that one of the key aspects of analyzing reinforcement learning algorithms is the use of Martingale convergence theorems, which is incorporated through the stochastic process in the system evolution. Since there is no result on improved Martingale convergence results in quantum computing so far (to the best of our knowledge), this work does a careful analysis to approach the result without use of Martingale convergence results.

Given the aforementioned quantum setup in place, this paper attempts to address the following:
{\em Can we design a quantum VI Algorithm that can improve classical  regret bounds in episodic RL setting?}

This paper answers the question in positive. The key to achieve such quantum advantage is the use of quantum environment that provides more information than just an observation of the next state. This enhanced information is used with quantum computing techniques to obtain efficient  reget bounds in this paper. 

To this end, we summarize the major contributions of our work as follows:
\begin{enumerate} %
	\item We present a novel quantum RL architecture that helps exploit the quantum advantage in episodic reinforcement learning.

	\item We propose QUCB-VI, which builds on the classical UCB-VI algorithm \citep{azar2017minimax}, wherein we carefully leverage available quantum information and quantum mean estimation techniques to engineer computation of agent's policy.  
	\item We perform rigorous theoretical analysis of the proposed framework and characterize its performance in terms of \textit{ regret} accumulated across $K$ episodes of agent's interaction with the unknown \textit{Markov Decision Process} (MDP) environment. More specifically, we show that QUCB-VI incurs $\Tilde{\mathcal{O}}(1)$  regret. We note that our algorithm provides a faster convergence rate in comparison to classical UCB-VI which accumulates $\Tilde{\mathcal{O}} (\sqrt{K})$ regret, where $K$ is the number of training  episodes.
	\item We conduct thorough experimental analysis of QUCB-VI (algorithm \ref{algo: UCBVI}) and compare against baseline classical UCB-VI algorithm on a variety of benchmark RL environments. Our experimental results reveals QUCB-VI's performance improvements in terms of regret growth over baseline.   
\end{enumerate}

The rest of the paper is organized as follows.  In Section \ref{sec: related_work}, we present a brief background of key existing literature pertaining to classical RL, as well as discuss prior research conducted in development of quantum mean estimation techniques and quantum RL methodologies relevant to our work.  In Section \ref{sec: problem_formln}, we mathematically formulate the problem of episodic RL in a finite horizon unknown MDP with the use of quantum oracles in the environment.  In Section \ref{sec: algo_framework}, we describe the proposed QUCB-VI Algorithm while bringing out key differences involving agent's policy computations as compared to classical UCB-VI. Subsequently, we provide the formal  analysis of  regret for the proposed algorithm in Section \ref{sec: theory_result}. In Section \ref{sec: experiments}, we report our results of experimental evaluations performed on various RL environments for the proposed algorithm and classical baseline method. Section \ref{sec: conclusion} concludes the paper.
\section{Background and Related Work} \label{sec: related_work}
\textbf{Classical reinforcement learning:} ~In the context of classical RL, an appreciable segment of prior research focus on obtaining theoretical results in \textit{tabular} RL, i.e., agent's state and action spaces are discrete \citep{sutton2018reinforcement}. Several existing methodologies guarantee sub-linear \textit{ regret} in this setting via leveraging \textit{optimism in the face of uncertainty} (OFU) principle \citep{lai1985asymptotically}, to strategically  balance \textit{exploration-exploitation} trade-off  \citep{osband2016generalization, strehl2006pac}. Furthermore, on the basis of design requirements and problem specific use cases, such algorithms have been mainly categorized as either \textit{model-based} ~\citep{auer2008near, dann2017unifying} or \textit{model-free} ~\citep{jin2018q, du2019provably}. In the episodic tabular RL problem setup, the optimal \textit{ regret} of $\Tilde{\mathcal{O}}(\sqrt{K})$ ($K$ is the number of episodes) have been studied for both \textit{model-based} as well as \textit{model-free} learning frameworks \citep{azar2017minimax, jin2018q}. In this paper, we study model-based algorithms and derive $\Tilde{\mathcal{O}}(1)$ regret with the use of quantum environment.

\textbf{Quantum Mean Estimation:} ~
Mean estimation is a statistical inference problem in which samples are used to produce an estimate of the mean of an unknown distribution. The improvement in sample complexity for mean estimation using quantum computing has been widely studied \citep{grover1998framework, brassard2002quantum, brassard2011optimal}. In \citep{montanaro2015quantum}, a quantum information assisted Monte-Carlo Algorithm was proposed which achieves an asymptotic near-quadratic faster convergence over its classical baseline. %
In this paper, we use the approach in \citep{hamoudi2021quantum} for the mean estimation of quantum random variables. We first describe the notion of random variable and corresponding extension to quantum random variable. 

\begin{definition}[Random Variable]
	A finite random variable is a function $X: \Omega \to E$ for some probability space $(\Omega, P)$, where $\Omega$ is a finite sample set, $P:\Omega\to[0, 1]$ is a probability mass function and $E\subset \mathbb{R}$ is the support of $X$. As is customary, we will often omit to mention $(\Omega, P)$ when referring to the random variable $X$.
\end{definition}

The notion is extended to a quantum random variable (or q.r.v.) as follows. 

\begin{definition}[Quantum Random Variable]
	A q.r.v. is a triple $(\mathcal{H},U,M)$ where $\mathcal{H}$ is a finite-dimensional Hilbert space, $U$ is a unitary transformation on $\mathcal{H}$, and $M = \{M_x\}_{x\in E}$ is a projective measurement on $\mathcal{H}$ indexed by a finite set $E \subset \mathbb{R}$. Given a random variable $X$ on a probability space $(\Omega, P)$, we say that a q-variable $(\mathcal{H}, U, M )$ generates $X$ when,
	\begin{itemize}
		\item[(1)] $\mathcal{H}$ is a finite-dimensional Hilbert space with some basis $\{|\omega \rangle\}_{\omega \in \Omega}$ indexed by $\Omega$.
		\item[(2)] $U$ is a unitary transformation on $\mathcal{H}$ such that $U|\mathbf{0}\rangle=\sum_{\omega \in \Omega} \sqrt{P(\omega)}|\omega\rangle$.
		\item[(3)] $M = \{M_x\}_{x}$ is the projective measurement on $\mathcal{H}$ defined by $M_x=\sum_{\omega: X(\omega)=x}|\omega\rangle\langle\omega|$.
	\end{itemize}
\end{definition}

We now define the notion of a quantum experiment.     Let $(\mathcal{H},U,M)$ be a q.r.v. that generates $X$. With abuse of notations, we call $X$ as the q.r.v. even though the actual q.r.v. is the $(\mathcal{H},U,M)$ that generates $X$. We define a quantum experiment as the process of applying any of the unitaries $U$, their
inverses or their controlled versions, or performing a measurement according to $M$. We also assume an access to the  quantum evaluation oracle $|\omega\rangle|0\rangle \to |\omega\rangle|X(\omega)\rangle$. Using this quantum oracle, the quantum mean estimation result can be stated as follows.

\begin{lemma}[Sub-Gaussian estimator \citep{hamoudi2021quantum}]
	\label{lem: SubGau}
	Let $X$ be a q.r.v. with mean $\mu$ and variance $\sigma^2$. Given  $n$ i.i.d. samples of q.r.v. $X$ and a real $\delta \in (0,1)$ such that $n > \log(1/\delta)$, a quantum algorithm \texttt{SubGaussEst}$(X,n,\delta)$ (please refer to algorithm 2 in \citep{hamoudi2021quantum}) outputs a mean estimate $\hat{\mu}$ such that,
	\begin{align}
		P\left[|\hat{\mu}-\mu|\le \frac{\sigma \log(1/\delta)}{n}\right]\ge 1-\delta.
	\end{align}
	The algorithm performs $O(n\log^{3/2}(n)\log\log(n))$ quantum experiments.
\end{lemma}
\vspace{-2mm}
We note that this result achieves the mean estimation error of $1/n$ in contrast to $1/\sqrt{n}$ for the classical mean estimation, thus providing a quadratic reduction in the number of i.i.d. samples needed for same error bound.

\textbf{Quantum reinforcement learning:} ~Recently, quantum mean estimation techniques have been applied with favorable theoretical convergence speed-ups for Quantum \textit{multi-armed bandits} (MAB) problem setting \citep{casale2020quantum,wang2021quantum,lumbreras2022multi}. However, bandits do not have the notion of state evolution like in reinforcement learning. Further, quantum reinforcement learning has been studied in  \citep{paparo2014quantum, dunjko2016quantum,dunjko2017advances, jerbi2021quantum,dong2008quantum}, while these works do not study the regret performance. The theoretical regret performance has been recently studied in \citep{wang2021quantum}, where a generative model is assumed and sample complexity guarantees are derived for discounted infinite horizon setup. In contrast, our work does not consider discounted case, and we don't assume a generative model. This paper demonstrates the quantum speedup for episodic reinforcement learning. 

\section{Problem Formulation} \label{sec: problem_formln}

\begin{figure}
	\centering
	\includegraphics[width=3.25in]{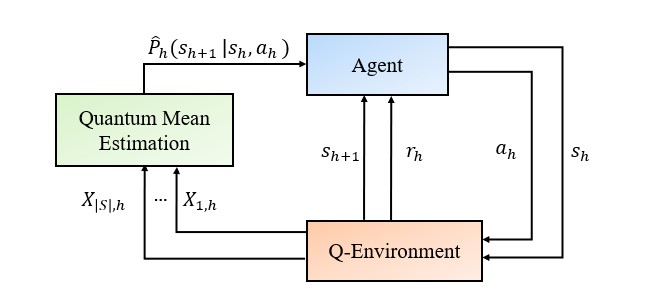}
	\caption{Quantum episodic reinforcement learning architecture depicting agent-environment interaction at round $h$.}
	\label{fig:qrl_setup}
\end{figure}

We consider episodic reinforcement learning in a finite horizon Markov Decision Process \citep{agarwal2019reinforcement} given by a tuple $\left(\mathcal{S}, \mathcal{A},H, \{P_{h}\}_{h \in [0, H-1]},
\{r_{h}\}_{h \in [0, H-1]}\right)$, where $\mathcal{S}$ and $\mathcal{A}$ are the state and the action spaces with cardinalities $S$ and $A$, respectively, $H \in \mathbb{N}$ is the episode length, $P_h(s^\prime|s,a) \in [0,1]$ is the probability of transitioning to state $s^\prime$ from state $s$ provided action $a$ is taken at step $h$ and $r_h(s,a) \in [0,1]$ is the  immediate reward associated with taking action $a$ in state $s$ at step $h$. In our setting, we denote an episode by the notation $k$, and every such episode comprises of $H$ rounds of agent's interaction with the learning environment. In our problem setting, we assume an MDP with a fixed start state $s_0$, where at the start of each new episode the state is reset to $s_0$. We note that the results can be easily extended to the case where starting state is sampled from some distribution. This is because we can have a dummy state $s_0$ which transitions to the next state $s_1$ coming from this distribution, independent of action, and having a reward of $0$. 

We encapsulate agent's interaction with the unknown MDP environment via the architecture as presented in Fig. \ref{fig:qrl_setup}. At an arbitrary time step $h$, given a state $s_h$ and action $a_h$, the environment gives the reward $r_h$ and next state $s_{h+1}$. Furthermore, we highlight that in our proposed architecture this set of signals i.e., $\{s_h, a_h, r_h, s_{h+1}\}$ are collected at the agent as \textit{classical} information. Additionally, our architecture facilitates availability of $S$  \textit{quantum random variables} (q.r.v.) ($X_{1,h}, \cdots, X_{S,h}$) at the agent's end, wherein q.r.v. $X_{i,h}$ generates the random variable $\mathbf{1}\left\{s_{h+1} = i\right\}$. This q.r.v. corresponds to the Hilbert space with basis vectors $|0\rangle$ and $|1\rangle$ and the unitary transformation, given as follows:
\begin{align}
	U|\mathbf{0}\rangle = & \sqrt{1-P_{h}(s_{h+1} = i|s_h,~a_h)} |0 \rangle \nonumber \\
	& ~~~+ \sqrt{P_{h}(s_{h+1} = i|s_h,~a_h)} |1 \rangle.  
\end{align}
We note that these q.r.v.'s can be generated by using a quantum next state from the environment which is given in form of the basis vectors for $S$ states as $|1 0 \cdots 0>$ with $S$ \textit{qubits} for first state and so on till $|0 \cdots 0 1>$ for the last state. Thus, the overall quantum next state is the superposition of these $S$ states with the amplitudes as $\sqrt{P_{h}(s_{h+1} = 1|s_h,~a_h)}, \cdots, \sqrt{P_{h}(s_{h+1} = S|s_h,~a_h)}$, respectively. The $S$ q.r.v.'s correspond to the $S$ \textit{qubits} in this next quantum state. Further, the next state can be obtained as a measurement of this joint next state superposition. Thus, assuming that the quantum environment can generate multiple copies of the next state superposition, all the $S$ q.r.v.'s and the next state measurement can be obtained. 

We note that the agent does not know $\{P_{h}\}_{h \in [0, H-1]}$, which needs to be estimated in the model-based setup. In order to estimate this, we will use the quantum mean estimation approach. This approach needs as quantum evaluation oracle $|\omega\rangle |0\rangle \to |\omega\rangle |\mathbf{1}\left\{s_{h+1} = \omega\right\}\rangle$ for $\omega \in \{0,1\}$. In Section \ref{sec: algo_framework}, we provide more details on how the aforementioned set of quantum indicator variables i.e., $\{{X_{s',h}}\}_{s' \in \mathcal{S}}$ are fed to a specific quantum mean estimation procedure to obtain the transition probability model.

Based on the observations, the agent needs to determine a policy $\pi_h$ which determines action $a_h$ given state $s_h$. For given policy $\pi$ and $h \in \{0,\cdots,H-1\}$, we define the value function $V_h^{\pi}: \mathcal{S} \rightarrow \mathbb{R}$ as
\begin{align}
	V_h^{\pi}(s)=\mathbb{E}\left[\sum_{t=h}^{H-1} r_h\left(s_t, a_t\right) \mid \pi, ~s_h = s\right],    
\end{align}
where the expectation is with respect to the randomness of the trajectory, that is, the randomness in state transitions and the stochasticity of $\pi$. Similarly, the state-action value (or $Q$-value) function $Q_h^\pi: \mathcal{S} \times \mathcal{A} \rightarrow \mathbb{R}$ is defined as 
\begin{align}
	Q_h^\pi(s, a)=\mathbb{E}\left[\sum_{t=h}^{H-1} r_h\left(s_t, a_t\right) \mid \pi, s_h=s, a_h=a\right].    
\end{align}
We also use the notation $V^\pi(s)=V_0^\pi(s)$. Given a state s, the goal of the agent is to find a policy $\pi^*$ that maximizes the value, i.e., the optimization problem the agent seeks to solve is: 
\begin{align}
	\max _\pi ~V^\pi(s).    
\end{align}
Define $Q_h^{\star}(s, a)=\sup _{\pi \in \Pi} ~Q_h^\pi(s, a)$ and $V_h^{\star}(s)=\sup _{\pi \in \Pi} ~V_h^\pi(s)$.  The agent aims to minimize the expected cumulative \textit{ regret} incurred across $K$ episodes:
\begin{align}
	\label{eq:reg}
	\texttt{Regret:}  ~\mathbb{E} \big[KV^{*}(s_0) - \sum_{k = 0}^{K-1}\sum_{h = 0}^{H-1} r(s_h^k, a_h^k) \big].
\end{align}
In the following section, we describe the proposed algorithm and analyze its  regret in Section 5. 

\section{Algorithmic Framework} \label{sec: algo_framework}

In this section, we describe in detail our quantum information assisted algorithmic framework to perform learning of unknown MDP under finite horizon episodic RL setting. In particular, we propose a quantum algorithm that incorporates \textit{model-based} episodic RL procedure originally proposed in the classical setting \citep{azar2017minimax, agarwal2019reinforcement}. In algorithm \ref{algo: UCBVI}, we present Quantum \textit{Upper Confidence Bound} - \textit{Value Iteration} (QUCB-VI) Algorithm which takes number of episodes $K$, length of an episode $H$ and confidence parameter $\delta$ as inputs. In the very first step, the count of visitations corresponding to every state-action pair $(s,a) \in \mathcal{S} \times \mathcal{A}$ are initialized to 0. Subsequently, at the beginning of each episode $K$, the value function estimates for the entire state-space $\mathcal{S}$, i.e., $\{\widehat{V}_{H}^k(s)\}_{s\in \mathcal{S}}$ are set to 0.

Next, for each time instant up to $H-1$, we  update the transition probability model. Here, we utilize the set of \textit{quantum random variable} (q.r.v.)  $\{X_{s',h}^{k}\}_{s' \in \mathcal{S}}$ as introduced in our quantum RL architecture presented in Figure \ref{fig:qrl_setup}. Recall that the elements of $\{X_{s',h}^{k}\}_{s' \in \mathcal{S}}$ are defined at each time step $h$ during episode $k$ as follows:   
\begin{align}
	X_{s',h}^{k} \triangleq \mathrm{1}[s_{h+1}^{k} = s']. \label{eq: X_shk_def}
\end{align}

We define $N_h^k(s, a)$ as the number of times $(s, a)$ is visited before episode $k$. More formally, we have
\begin{align}
	N_h^k(s, a)=\sum_{i=0}^{k-1} \mathbf{1}\big[(s_h^i, a_h^i)=(s, a) \big]. \label{eq: n_hk_def}  
\end{align}

\begin{algorithm}[t!]
\caption{QUCB-VI}
\label{algo: UCBVI}
\begin{algorithmic}[1]
	\State \textbf{Inputs:} $K$, $H$, $\delta \in (0,1]$.
	\State Set ${N}_h^1(s,a) \leftarrow 0, ~\forall ~s \in \mathcal{S}, ~a \in \mathcal{A},~ h \in [0,H-1])$.
	\For {$k=1,\dots,K$}
	\State Set $\widehat{V}_{H}^k(s) \leftarrow 0, ~\forall s \in \mathcal{S}$.
	\For {$h=H-1,\dots,1,0$}.
	\State Set  $\{\hat{P}_{h}^{k}(s' | s,a )\}_{s,a,s') \in \mathcal{S} \times \mathcal{A} \times \mathcal{S}}$ via Eq. \eqref{eq: sub_gaussian_txprob_update}.
	\State Set $b_h^k(s,a) \leftarrow \frac{ \log(SAHK/\delta)}{N_h^k(s,a)}$.
	\State Set $\{\widehat{Q}_h^k(s,a)\}_{s \in \mathcal{S}, a \in \mathcal{A}}$ via Eq. \eqref{eq: Q-func-update}.
	\State Set $\{\widehat{V}_{h}^k(s)\}_{s \in \mathcal{S}}$ via Eq. \eqref{eq: Value_update}.
	\EndFor
	\State Set policies $\{\pi_h^k(s)\}_{s \in \mathcal{S}, h \in [0, H-1]}$ using Eq. \eqref{eq: policy_calc}.
	\State \hspace{-2mm} Get Trajectory $\{s_{h}^{k}, a_{h}^{k}, r_{h}^{k}, s_{h+1}^{k}\}_{h=0}^{H-1}$ via  $\{\pi_h^k\}_{h=0}^{H-1}$.
	\State Set $\{{N}_h^{k+1}(s,a)\}_{s \in \mathcal{S}, a \in \mathcal{A}, h \in [0,H-1]}$ via Eq. \eqref{eq: n_hk_def}.  
	\EndFor
	
\end{algorithmic}
\end{algorithm}

${N}_h^k(s,a)$ indicates the number of samples obtained for $(s,a)$ in the past which will help in efficient averaging to estimate the transition probabilities. With the formulation of q.r.v. $\{X_{s',h}^{k}\}_{s' \in \mathcal{S}}$ in place, we update the transition probability model elements, i.e., $\{\hat{P}_{h}^{k}(s' | s,a )\}_{(s,a,s') \in \mathcal{S} \times \mathcal{A} \times \mathcal{S}}$, in step 6 of algorithm \ref{algo: UCBVI}. To get the estimate of $\hat{P}_{h}^{k}(s' | s,a )$, we use the ${N}_h^k(s,a)$ q.r.v.'s $X_{s',h}^{k}$ for all past $k$. Thus,  $\hat{P}_{h}^{k}(s' | s,a )$ is estimated as:
\begin{align}
\hat{P}_{h}^{k}(s' | s,a ) \leftarrow ~\texttt{SubGaussEst}(X_{s',h}^{k}, {N}_h^k(s,a) ,\delta), \label{eq: sub_gaussian_txprob_update}
\end{align}
where subroutine \texttt{SubGaussEst} as presented in Algorithm 2 of \citep{hamoudi2021quantum} performs mean estimation of q.r.v. $X_{s',h}^{k}$ given ${N}_h^k(s,a)$ collection of samples and confidence parameter $\delta$. We emphasize that step 6 in Algorithm \ref{algo: UCBVI} brings out the key change w.r.t. classical UCB-VI via carefully estimating mean of quantum information collected by agent via interaction with the unknown MDP environment.  In step 7, we set the reward bonus i.e., $b_h^k(s,a)$ which resembles a Bernstein-style UCB bonus, essentially inducing \textit{optimism} in the learnt model. Consequently, step 8-9 compute estimates of $\{{Q}_h^k(s,a)\}_{s \in \mathcal{S}, a \in \mathcal{A}}, ~\{\widehat{V}_{h}^k(s)\}_{s \in \mathcal{S}}$ by adopting the following \textit{Value Iteration} based updates at time step $h$:
\begin{align}
& \widehat{Q}_h^k(s,a) \leftarrow
\min\{H,{r}_{h}^{k}(s, a) 
+ \langle \widehat{V}_{h+1}^k, \widehat{P}_{h}^{k}\left(\cdot|s, a\right) \rangle \nonumber \\
& \hspace{3cm}  + b_h^k(s,a) \},  \label{eq: Q-func-update} \\
& \widehat{V}_{h}^k(s) \leftarrow \underset{a \in \mathcal{A}}{\max} ~\widehat{Q}_h^k(s,a). \label{eq: Value_update} 
\end{align}
This \textit{Value Iteration} procedure (i.e., inner loop consisting of steps 6-9) is executed for $H$ time steps thereby generating a collection of $H$ policies $\{\pi_h^k(s)\}_{s \in \mathcal{S}, h \in [0, H-1]}$ calculated for each pair of $(s,h)$ in step 11 as:
\begin{align}
\pi_h^k(s) \leftarrow \underset{a \in \mathcal{A}}{\arg\max} ~\widehat{Q}_h^k(s,a), \label{eq: policy_calc} 
\end{align}
Next, using the updated policies i.e., $\{\pi_h^k(s)\}_{s \in \mathcal{S}, h \in [0, H-1]}$ which are based on observations recorded till episode $k-1$, the agent collects a new trajectory of $H$ tuples pertaining to episode $k$ i.e., $\{s_{h}^{k}, a_{h}^{k}, r_{h}^{k}, s_{h+1}^{k}\}_{h=0}^{H-1}$ in step 12 starting from initial state reset to $s_0$. Finally, the frequency of agent's visitation to all state action pairs at every time step $h$ over the $k$ episodes i.e., $\{{N}_h^{k+1}(s,a)\}_{s \in \mathcal{S}, a \in \mathcal{A}, h \in [0,H-1]}$ are updated in step 13. Consequently, algorithm \ref{algo: UCBVI} triggers a new episode of agent's interaction with the unknown MDP environment.

\section{Regret Results for the Proposed Algorithm} \label{sec: theory_result}

\subsection{Main Result: Regret Bound for QUCB-VI}
In Theorem \ref{thm: main_regr}, we present the cumulative \textit{ regret} collected upon deploying QUCB-VI in an unknown MDP environment $\mathcal{M}$ (please refer to Section \ref{sec: problem_formln} for definition of MDP)  over a finite horizon of $K$ episodes.
\begin{theorem} \label{thm: main_regr} In an unknown MDP environment $\mathcal{M} \triangleq \left(\mathcal{S}, \mathcal{A},H, \{P_{h}\}_{h \in [0, H-1]},
\{r_{h}\}_{h \in [0, H-1]}\right)$, the  regret incurred by QUCB-VI (algorithm \ref{algo: UCBVI}) across $K$ episodes is bounded as follows:
\begin{align}
	\mathbb{E} \Big[\sum_{k=0}^{K-1} \big(V^{*}(s_0) &- V^{\pi^{k}}(s_0) \big)  \Big] \nonumber \\
	& \leq O (H^2 S^2 A \log^2(SAH^2 K^2 )). \label{eq: _regr}
\end{align}
\end{theorem}

The result obtained via Eq. \eqref{eq: _regr} in Theorem \ref{thm: main_regr} brings out the key advantage of the proposed framework in terms of accelerating the \textit{ regret} convergence rate to $\Tilde{\mathcal{O}}(1)$ against the classical result of $\Tilde{\mathcal{O}} (\sqrt{K})$ \citep{azar2017minimax}.

In order to prove Theorem \ref{thm: main_regr}, we present the following auxiliary mathematical results in the ensuing subsections: bound for probability transition model error pertaining to every state-action pair (section \ref{sec: model_err_SA}); optimism exhibited by the learnt model understood in terms of value functions of the states (section \ref{sec: model optimism}); a supporting result that bounds inverse frequencies of state-action pairs over the entire observed trajectory (section \ref{sec: traj_bound}). Subsequently, we utilize these aforementioned theoretical results to prove Theorem \ref{thm: main_regr} 
in section \ref{sec: thm_proof}.

\subsection{Probability Transition Model Error for state-action pairs} \label{sec: model_err_SA}
\begin{lemma} \label{lemma: model_err_SA} For $k \in \{0, \ldots, {K} - 1 \}$, $s \in \mathcal{S}$, $a \in \mathcal{A}$, $h \in \{0, \ldots, H - 1 \}$, for any $f : \mathcal{S} \rightarrow [0, H]$, execution of QUCB-VI (algorithm \ref{algo: UCBVI}) guarantees that the following holds with probability at least $1-\delta$:
\begin{align}
	\Big|\Big(\hat{P}_{h}^{k}(\cdot | s,a ) - {P}_{h}(\cdot | s,a ) \Big)^{T} f \Big| \leq \frac{HSL}{N_{h}^{k}(s,a)}, \label{eq: model_est_err}
\end{align}
where $L \triangleq {\log (SAHK/\delta)}$ and $\{ N_{h}^{k}(s,a) $, $ \hat{P}_{h}^{k}(s' | s,a ) \}$ are defined in Eq. \eqref{eq: n_hk_def}, \eqref{eq: sub_gaussian_txprob_update}.
\end{lemma}
\begin{proof}
To prove the claim in Eq. \eqref{eq: model_est_err}, we consider an arbitrary tuple $(s,a,k,h,f)$ and obtain the following:
\begin{align}
	\Big|\Big(\hat{P}_{h}^{k}(\cdot | s,a )  & - {P}_{h}(\cdot | s,a ) \Big)^{T} f \Big| \nonumber \\
	& \leq \sum_{s' \in \mathcal{S}} f(s')|\hat{P}_{h}^{k}(s' | s,a ) - {P}_{h}(s' | s,a )|, \label{eq: model_est_err_1} \\
	& \leq H \sum_{s' \in \mathcal{S}} |\underbrace{\hat{P}_{h}^{k}(s' | s,a ) - {P}_{h}(s' | s,a )}_{\text{(a)}}|, \label{eq: model_est_err_2}
\end{align}
where Eq. \eqref{eq: model_est_err_2} is due to the definition of $f(\cdot)$ as presented in the lemma statement. Next, in order to analyze (a) in Eq. \eqref{eq: model_est_err_2}, we note the definition of $X_{s',h}^{k}$ presented in Eq. \eqref{eq: X_shk_def}  as an \textit{indicator} q.r.v allows us to write:
\begin{align}
	& \hat{P}_{h}^{k}(s' | s,a ) = \mathbb{E}_{q, \hat{P}_{h}^{k}}[X_{s',h}^{k}|s,a], \label{eq: model_est_err_3} \\
	& {P}_{h}(s' | s,a ) = \mathbb{E}_{q, {P}_{h}}[X_{s',h}^{k}|s,a]. \label{eq: model_est_err_4}
\end{align}
Using the fact that $X_{s',h}^{k}$ is a q.r.v. allows us to directly apply Lemma \ref{lem: SubGau}, thereby further bounding (a) in Eq. \eqref{eq: model_est_err_2} with probability at least $1- \delta$ as follows:
\begin{align}
	|\hat{P}_{h}^{k}(s' | s,a ) & - {P}_{h}(s' | s,a )| \nonumber \\
	& \leq \frac{\log (1/\delta)}{N_{h}^{k}(s,a)}, \\
	& \leq \frac{\log(SAHK/\delta)}{N_{h}^{k}(s,a)}
	= \frac{L}{N_{h}^{k}(s,a)}, \label{eq: model_est_err_5}
\end{align}
where, we emphasize that Eq. \eqref{eq: model_est_err_5} is a consequence of applying union-bound $\forall s,a,h,k$ as well as we use the definition of $L$ provided in the lemma statement. Plugging the bound of term (a) as obtained in Eq. \eqref{eq: model_est_err_5} back into Eq. \eqref{eq: model_est_err_2}, we obtain:
\begin{align}
	\Big|\Big(\hat{P}_{h}^{k}(\cdot | s,a ) - {P}_{h}(\cdot | s,a ) \Big)^{T} f \Big| & \leq H \sum_{s' \in \mathcal{S}} \frac{L}{N_{h}^{k}(s,a)}, \\
	& \leq \frac{HSL}{N_{h}^{k}(s,a)},
\end{align}
which proves the claim of the Lemma.
\end{proof}
\textbf{Interpretation of Lemma \ref{lemma: model_err_SA}:} ~One of the key insights that we draw from this lemma is that the quantum mean estimation with q.r.v. $X_{s',h}^{k}$ allowed the use of Lemma \ref{lem: SubGau}, which facilitated quadratic speed-up of transition probability model convergence. More specifically, our result suggests transition probability model error diminishes with $\Tilde{\mathcal{O}} (\frac{1}{N_{h}^{k}(s,a)})$ speed as opposed to the classical results of $\Tilde{\mathcal{O}} (\frac{1}{\sqrt{N_{h}^{k}(s,a)}})$ \citep{azar2017minimax, agarwal2019reinforcement}.  
\subsection{Optimistic behavior of QUCB-VI} \label{sec: model optimism}
\begin{lemma} \label{lemma: model optimism} Assume that the event described in Lemma \ref{lemma: model_err_SA} is true. Then, the following holds $\forall k$:
\begin{align}
	\widehat{V}_{0}^{k}(s) \geq {V}_{0}^{*}(s), ~\forall s \in \mathcal{S}, 
\end{align}
where $\widehat{V}_{0}^{k}(s)$ is calculated via our QUCB-VI Algorithm and ${V}_{h}^{*}: \mathcal{S} \rightarrow [0,H]$.
\end{lemma}
\begin{proof}
To prove the lemma statement, we proceed via mathematical induction. Firstly, we highlight that the following holds at time step $H$:
\begin{align}
	\hat{V}_{H}^{k}(s) = {V}_{H}^{*}(s) = 0, ~\forall s \in \mathcal{S},
\end{align}
In the next step, assume that $\hat{V}_{h+1}^{k}(s) \geq {V}_{h+1}^{*}(s)$. If $\hat{Q}_{h}^{k}(s,a) = H$, then $\hat{Q}_{h}^{k}(s,a) \geq {Q}_{h}^{*}(s,a)$ since ${Q}_{h}^{*}(s,a)$ can be atmost $H$. Otherwise, at time step $h$, we obtain:
\begin{align}
	& \hat{Q}_{h}^{k}    (s,a) - {Q}_{h}^{*}(s,a)  \nonumber \\
	& = b_{h}^{k}(s,a) +  \langle \hat{P}_{h}^{k}(\cdot | s,a), \hat{V}_{h}^{k} \rangle - \langle {P}_{h}(\cdot | s,a), {V}_{h}^{*} \rangle,\\
	& \geq b_{h}^{k}(s,a) +  \langle \hat{P}_{h}^{k}(\cdot | s,a), {V}_{h}^{*} \rangle - \langle {P}_{h}(\cdot | s,a), {V}_{h}^{*} \rangle, \label{eq: model optimism eq 1} \\
	& = b_{h}^{k}(s,a) \nonumber \\
	& ~~~+ \sum_{s' \in \mathcal{S}}\big(  \hat{P}_{h}^{k}(s' | s,a) - {P}_{h}(s' | s,a) \big){V}_{h}^{*}(s'), \\
	& \geq b_{h}^{k}(s,a) - \frac{HSL}{N_{h}^{k}(s,a)}, \label{eq: model optimism eq 2} \\
	& = 0, \label{eq: model optimism eq 3}
\end{align}
where Eq. \eqref{eq: model optimism eq 1} is due to the induction assumption. Furthermore, Eq. \eqref{eq: model optimism eq 2}, \eqref{eq: model optimism eq 3}  are owed to Lemma \ref{lemma: model_err_SA} and definition of bonus in step 7 of Algorithm \ref{algo: UCBVI}, respectively. Hence, we have $\hat{Q}_{h}^{k}(s,a) \geq {Q}_{h}^{*}(s,a)$. Using \textit{Value Iteration} computations in Eq. \eqref{eq: Q-func-update} - \eqref{eq: Value_update},  we obtain $\hat{V}_{h}^{k}(s) \geq {V}_{h}^{*}(s), ~\forall h$.

This completes the proof.
\end{proof}
\textbf{Interpretation of Lemma \ref{lemma: model optimism}:} ~This Lemma reveals that QUCB-VI (algorithm \ref{algo: UCBVI}) outputs estimates of the value function which are always lower bounded by the true value at each time step, thereby exhibiting similar optimistic behavior as the classical UCB-VI algorithm. Interestingly, faster convergence properties of QUCB-VI's transition model (i.e., Lemma \ref{lemma: model_err_SA}) complemented the usage of a sharper  bonus term (i.e., $b_{h}^{k}(s,a)$ defined in algorithm \ref{algo: UCBVI}) instead of  the bonus terms of the classical algorithm, while keeping the optimism behavior of the model intact.
\subsection{Trajectory Summation Bound Characterization} \label{sec: traj_bound}
Next, we present a technical result bounding inverse of observed state-action pair frequencies over agent's trajectory collected across all the episodes in Lemma \ref{lemma: trajectory_sum}.
\begin{lemma} \label{lemma: trajectory_sum}
Assume an arbitrary sequence of trajectories $\{s_{h}^{k}, a_{h}^{k} \}_{h = 0}^{H-1}$ for $k = 0, \cdots, K-1$. Then, the following result holds:
\begin{align}
	\sum_{k = 0}^{K-1} \sum_{h=0}^{H-1} \frac{1}{N_{h}^{k}(s_h^k, a_h^k)} \leq HSA \log (K).
\end{align}
\end{lemma}
\vspace{-1cm}
\begin{proof}
We change order of summations to obtain:
\begin{align}
	\sum_{k = 0}^{K-1} &\sum_{h = 0}^{H-1} \frac{1}{N_{h}^{k}(s_h^k, a_h^k)} \nonumber \\
	&= \sum_{h = 0}^{H-1} \sum_{k = 0}^{K-1} \frac{1}{N_{h}^{k}(s_h^k, a_h^k)}, \\
	& = \sum_{h = 0}^{H-1} \sum_{(s,a) \in \mathcal{S} \times \mathcal{A}} \sum_{i = 1}^{N_{h}^{K}(s,a)} \frac{1}{i}, \\
	& \leq \sum_{h = 0}^{H-1} \sum_{(s,a) \in \mathcal{S} \times \mathcal{A}} \log (N_{h}^{K}(s,a)), \label{eq: trajectory_sum_1} \\
	& \leq HSA ~{\max}_{\mathcal{S} \times \mathcal{A} \times [0,H-1]} \log (N_{h}^{K}(s,a)), \\
	& \leq HSA \log (K),
\end{align}
where Eq. \eqref{eq: trajectory_sum_1} is due to the fact that $\sum_{i=1}^{N} 1/i \leq \log (i)$.

This completes the proof of the lemma statement.
\end{proof}

\subsection{Proof of Theorem \ref{thm: main_regr}} \label{sec: thm_proof}
To prove Theorem \ref{thm: main_regr}, we first note that the following holds for episode $k$:
\begin{align}
V^{*}(s_0) & - V^{\pi^{k}}(s_0) \nonumber \\
&\leq \hat{V}_{0}^{k}(s_0) - V_{0}^{\pi^{k}}(s_0), \label{eq: sim_lemma 1} \\
& = \hat{Q}_{0}^{k}(s_0, \pi^{k}(s_0)) - {Q}_{0}^{\pi^{k}}(s_0, \pi^{k}(s_0)), \\
& \leq r_{0}^{k}(s_0, \pi^{k}(s_0)) + b_{0}^{k}(s_0, \pi^{k}(s_0))   \nonumber \\
& ~~~+ \langle \widehat{V}_{1}^k, \widehat{P}_{0}^{k}\left(\cdot|s_0, \pi^{k}(s_0) \right) \rangle - r_{0}^{k}(s_0, \pi^{k}(s_0)) \nonumber \\
& ~~~- \langle {V}_{1}^{\pi^{k}},{P}_{0}\left(\cdot|s_0, \pi^{k}(s_0) \right) \rangle, \label{eq: sim_lemma 2} \\
& = b_{0}^{k}(s_0, \pi^{k}(s_0)) + \langle \widehat{V}_{1}^k,\widehat{P}_{0}^{k}\left(\cdot|s_0, \pi^{k}(s_0) \right) \rangle \nonumber \\
& ~~~- \langle {V}_{1}^{\pi^{k}},{P}_{0}\left(\cdot|s_0, \pi^{k}(s_0) \right) \rangle, \\
& = b_{0}^{k}(s_0, \pi^{k}(s_0)) \nonumber \\
& ~~~+ \langle \widehat{V}_{1}^k,\widehat{P}_{0}^{k}\left(\cdot|s_0, \pi^{k}(s_0) \right) - {P}_{0}\left(\cdot|s_0, \pi^{k}(s_0)\right) \rangle \nonumber \\
& ~~~+ \underbrace{\langle \widehat{V}_{1}^k - {V}_{1}^{\pi^{k}}, {P}_{0}\left(\cdot|s_0, \pi^{k}(s_0) \right) \rangle}_{\text{(a)}} \label{eq: sim_lemma 3} \\ 
& = \sum_{h=0}^{H-1} \mathbb{E}_{s_h, a_h \sim d_{h}^{\pi^k}}\Big[b_{h}^{k}(s_h,a_h) \nonumber \\
& ~~~+ \underbrace{\langle  \hat{P}_{h}^{k}(\cdot | s_h,a_h) - {P}_{h}(\cdot | s_h,a_h) , \hat{V}_{h+1}^{\pi^{k}} \rangle}_\text{(b)} \Big], \label{eq: main_regr_eq 1} 
\end{align}
where Eq. \eqref{eq: sim_lemma 1} is due to Lemma \ref{lemma: model optimism}. Eq. \eqref{eq: sim_lemma 2} uses the definition of $Q$ function in Eq. \eqref{eq: Q-func-update} and the fact that $\min \{a,b \} \leq a$. In Eq. \eqref{eq: sim_lemma 3}, we identify that term (a) is the 1-step recursion of RHS of Eq. \eqref{eq: sim_lemma 1}. Consequently, we obtain Eq. \eqref{eq: main_regr_eq 1} where the expectation is w.r.t. the intermediate trajectories $d_{h}^{\pi^k}$ generated via policies $\{ \pi_{i}^{k}\}_{i = 0}^{h-1}$.

Let us denote the event described in Lemma \ref{lemma: model_err_SA} as $\mathcal{E}$. Formally, $\mathcal{E}$ is described as follows for some $\delta \in [0,1]$:
\begin{align}
\mathcal{E} \triangleq \mathrm{1}\Big[|\hat{P}_{h}^{k}(s' | s,a ) & - {P}_{h}(s' | s,a )| \leq \frac{\log (1/\delta)}{N_{h}^{k}(s,a)}, ~\forall s,a,s' \Big]. \nonumber
\end{align}
Then, assuming that $\mathcal{E}$ is true, term (b) in Eq. \eqref{eq: main_regr_eq 1} can be further bounded as follows:
\begin{align}
\langle  \hat{P}_{h}^{k}&(\cdot |  s_h,a_h)  - {P}_{h}(\cdot | s_h,a_h) , \hat{V}_{h+1}^{\pi^{k}} \rangle \nonumber \\
& \leq \|\hat{P}_{h}^{k}(\cdot | s_h,a_h) - {P}_{h}(\cdot | s_h,a_h) \|_{1} \| \hat{V}_{h+1}^{\pi^{k}}\|_{\infty}, \label{eq: main_regr_eq 2} \\
& \leq \frac{HSL}{N_{h}^{k}(s,a)}, \label{eq: main_regr_eq 3}
\end{align}
where Eq. \eqref{eq: main_regr_eq 2}, \eqref{eq: main_regr_eq 3}  are consequences of Holder's inequality and Lemma \ref{lemma: model_err_SA} respectively. By plugging  the bound for term (a) as obtained in Eq. \eqref{eq: main_regr_eq 3} back into RHS of Eq. \eqref{eq: main_regr_eq 1}, we have:
\begin{align}
V^{*} & (s_0)  - V^{\pi^{k}}(s_0) \nonumber \\
& \leq \sum_{h=0}^{H-1} \mathbb{E}_{s_h, a_h \sim d_{h}^{\pi^k}}\Big[b_{h}^{k}(s_h,a_h) + \frac{HSL}{N_{h}^{k}(s,a)} \Big],  \label{eq: thm_val_diff1} \\
& \leq  \sum_{h=0}^{H-1} \mathbb{E}_{s_h, a_h \sim d_{h}^{\pi^k}}\Big[ 2 \frac{HSL}{N_{h}^{k}(s,a)}\Big], \label{eq: main_regr_eq 4} \\
& = 2 HSL \cdot \mathbb{E}\Big[ \sum_{h=0}^{H-1} \frac{1}{N_{h}^{k}(s,a)} \big| \mathcal{H}_{<k} \Big], \label{eq: main_regr_eq 5}
\end{align}
where Eq. \eqref{eq: main_regr_eq 4} is owed to the definition of  bonus $b_{h}^{k}$ in algorithm 1. Further, in Eq. \eqref{eq: main_regr_eq 5}, the expectation is w.r.t. trajectory $\{s_h^{k}, a_{h}^{k}\}$ generated via policy $\pi^{k}$ while conditioning on history collected till end of episode $k-1$, i.e., $\mathcal{H}_{<k}$. Summing up across all the episodes and taking into account success/failure of event $\mathcal{E}$, we obtain the following:
\begin{align}
\mathbb{E} & \Big[\sum_{k = 0}^{K-1} V^{*}(s_0) - V^{\pi^{k}}(s_0) \Big] \nonumber \\
& = \mathbb{E} \Big[ \mathrm{1}[\mathcal{E}] \Big(\sum_{k = 0}^{K-1} V^{*}(s_0) - V^{\pi^{k}}(s_0) \Big)\Big] \nonumber \\
& ~~~+ \mathbb{E} \Big[\mathrm{1}[\overline{\mathcal{E}}] \Big(\sum_{k = 0}^{K-1} V^{*}(s_0) - V^{\pi^{k}}(s_0) \Big) \Big], \\
& \leq \mathbb{E} \Big[ \mathrm{1}[\mathcal{E}] \Big(\sum_{k = 0}^{K-1} V^{*}(s_0) - V^{\pi^{k}}(s_0) \Big)\Big] + 2\delta KH, \label{eq: main_regr_eq 6} \\
& \leq 2 HSL \cdot \mathbb{E}\Big[ \sum_{k = 0}^{K-1} \sum_{h=0}^{H-1} \frac{1}{N_{h}^{k}(s,a)} \Big] + 2\delta KH, \label{eq: main_regr_eq 7} \\
& \leq 2 H^{2}S^{2}A L \log (K) + 2\delta KH, \label{eq: main_regr_eq 8}
\end{align}
where Eq. \eqref{eq: main_regr_eq 6} is owed to the facts that value functions $\{V^{*}, V^{\pi^{k}} \}$ are bounded by $H$ and the failure probability is at most $\delta$. Next, we obtain Eq. \eqref{eq: main_regr_eq 7} by leveraging Eq. \eqref{eq: main_regr_eq 5} when event $\mathcal{E}$ is successful. Eq. \eqref{eq: main_regr_eq 8} is a direct consequence of Lemma \ref{lemma: trajectory_sum}.
\begin{figure*}[htb!]
\centering
\subfigure[ \scriptsize Riverswim-6. $S = 6, ~A = 2.$]{
	\includegraphics[scale = 0.34]{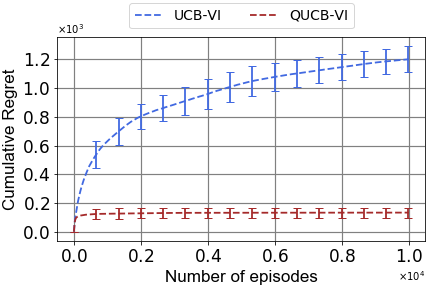}
	\label{fig:riverswim6_regr}
}
\subfigure[  \scriptsize Riverswim-12. $S = 12, ~A = 2.$]{
	\includegraphics[scale = 0.34]{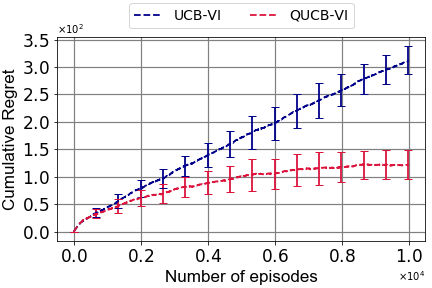}
	\label{fig:riverswim12_regr}
}
\subfigure[ \scriptsize  Grid-world. $S = 20, ~A = 4.$ ]{
	\includegraphics[scale = 0.34]{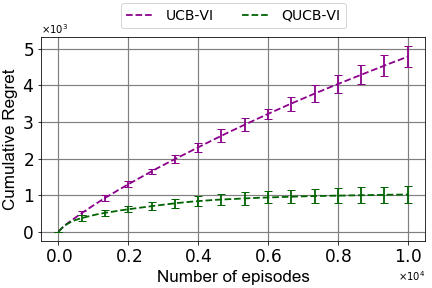}
	\label{fig: gridworld_regr}
}    
\caption{Cumulative Regret incurred by QUCB-VI (algorithm \ref{algo: UCBVI}) and UCB-VI algorithm for various RL environments.}
\label{fig: regret_fig}
\vspace{0in}
\end{figure*}

By setting, $\delta = 1/KH$, we obtain:
\begin{align}
\mathbb{E} & \Big[\sum_{k = 0}^{K-1} V^{*}(s_0) - V^{\pi^{k}}(s_0) \Big] \nonumber \\
& \leq 2H^2 S^2 A \log^2(SAH^2 K^2 ) + 2, \\
& = O (H^2 S^2 A \log^2(SAH^2 K^2 )).
\end{align}

This completes the proof of the theorem.

The choice of improved reward bonus term  manifests itself in Eq. \eqref{eq: thm_val_diff1}-\eqref{eq: main_regr_eq 5} and plays a significant role towards QUCB-VI's overall \textit{regret} improvement. Further, we note that the Martingale style proof approach and the corresponding Azuma Hoeffding's inequality, which are typical in the regret bound proof of classical UCB-VI,  are not used in the analysis of QUCB-VI algorithm.
\section{Numerical Evaluations} \label{sec: experiments}

In this section, we analyze the performance of QUCB-VI (Algorithm \ref{algo: UCBVI}) via proof-of-concept experiments on multiple RL environments. Furthermore, we investigate the viability of our
methodology against its classical counterpart UCB-VI \citep{azar2017minimax, agarwal2019reinforcement}. To this end, we first conduct our empirical evaluations on \textit{RiverSwim-6} environment comprising of 6 states and 2 actions, which is an extensively used environment for benchmarking \textit{model-based} RL frameworks \citep{osband2013more, tossou2019near, chowdhury2022differentially}. Next, we extend our testing setup to include \textit{Riverswim-12} with 12 states and 2 actions. Finally, we construct a \textit{Grid-world} environment \citep{sutton2018reinforcement} comprising of a $7 \times 7$ sized grid and characterized by 20 states and 4 actions.

\textbf{Simulation Configurations:} ~In our experiments, for all the aforementioned environments, we conduct training across $K = 10^4$ episodes, and every episode consists of $H=20$ time-steps. The environment is reset to a fixed initial state at the beginning of each episode. Furthermore, we perform 20 independent Monte-Carlo simulations and collect episode-wise cumulative regret incurred by QUCB-VI, and baseline UCB-VI algorithms. In our implementation of QUCB-VI, we accumulate the estimates of state transition probability model based on uniform sample from the actual transition probability within a window governed by the quantum mean estimation error.

\if 0
\begin{remark}
We highlight that classical implementation of QUCB-VI is facilitated by the closed form of measurement output distribution pertaining to quantum amplitude estimation procedures as corroborated by Theorem 11 in \citep{brassard2002quantum}. Consequently, this simulation strategy has been adopted in prior works in QRL ~\citep{wan2022quantum}.
\end{remark}
\fi 

\textbf{Interpretation of Results} ~In Figure \ref{fig: regret_fig}, we report our experimental results in terms of cumulative regret of agents rewards incurred against number of training episodes for each RL environment. In Fig. \ref{fig:riverswim6_regr}- \ref{fig: gridworld_regr}, we note that QUCB-VI significantly outperforms classical UCB-VI with a noticeable margin, while QUCB-VI also achieves model convergence within the chosen number of training episodes. These observations support the performance gains in terms of convergence speed of the proposed algorithm as revealed in our theoretical analysis of regret. In Fig. \ref{fig:riverswim12_regr}-\ref{fig: gridworld_regr}, we observe that classical UCB-VI suffers an increasingly linear trend in regret growth. This
indicates that in environments such as \textit{RiverSwim-12}, \textit{Grid-world} which are characterized by large diameter MDPs, it is necessary to increase training episodes in order to ensure sufficient exploration by the RL agent with classical environment. This demonstrates that quantum computing helps in significantly faster convergence.
\section{Conclusion and Future Work} \label{sec: conclusion}
We propose a Quantum information assisted \textit{model-based} RL methodology that  facilitates an agent's learning in an unknown MDP environment. To this end, we first present a carefully engineered architecture modeling agent's interaction with the environment at every time step. Consequently, we outline QUCB-VI Algorithm that suitably incorporates an efficient quantum mean estimator, leading to exponential theoretical convergence speed improvements in contrast to classical UCB-VI proposed in \citep{azar2017minimax}. Finally, we report evaluations on a set of benchmark RL environment which support the efficacy of QUCB-VI algorithm. As a future work, %
it will be worth exploring  whether the benefits can be translated to \textit{model-free} as well as continual RL  settings.

\bibliography{uai2023-template}

\end{document}